\documentclass[10pt,journal,compsoc]{IEEEtran}
\usepackage{anyfontsize}

\usepackage{amssymb}
\usepackage{amsmath}
\usepackage[linesnumbered,ruled,vlined]{algorithm2e}
\usepackage{CJK}
\usepackage{url}
\usepackage{graphicx}

\usepackage[caption=false]{subfig}
\usepackage{xcolor}
\usepackage{float}

\usepackage{amsfonts}

\usepackage{multirow}

\usepackage{booktabs}

\usepackage{color, xcolor}

\usepackage{soul}

\usepackage{mathrsfs} 

\sethlcolor{red}

\usepackage{bbding}
\newtheorem{myTheo}{Theorem}
\newtheorem{proof}{Proof}
\newtheorem{definition}{Definition}

\ifCLASSOPTIONcompsoc

  \usepackage[nocompress]{cite}
\else

  \usepackage{cite}
\fi

\ifCLASSINFOpdf

\else

\fi

\begin{document}

\title{Directed Link Prediction using GNN \\
with Local and Global Feature Fusion 
}

\author{Yuyang~Zhang, Xu~Shen,
	Yu~Xie,	Ka-Chun~Wong, Weidun Xie, and Chengbin~Peng$^*$
\IEEEcompsocitemizethanks{\IEEEcompsocthanksitem $\ \ $Yuyang Zhang, Xu Shen, Yu Xie and Chengbin Peng are with College of Information Science and Engineering, Ningbo University, Ningbo, 315211, China.\protect
	\IEEEcompsocthanksitem $\ \ $Ka-Chun Wong and Weidun Xie are with City University of Hong Kong, Hong Kong, 999077, China.\protect\\
	$*$Corresponding author: pchbin@gmail.com}	
}

\markboth{Journal of \LaTeX\ Class Files,~Vol.~14, No.~8, August~2015}
{Shell \MakeLowercase{\textit{et al.}}: Bare Demo of IEEEtran.cls for Computer Society Journals}

\IEEEtitleabstractindextext{
\begin{abstract}
Link prediction is a classical problem in graph analysis with many practical applications. For directed graphs, recently developed deep learning approaches typically analyze node similarities through contrastive learning and aggregate neighborhood information through graph convolutions. 
In this work, we propose a novel graph neural network (GNN) framework to fuse feature embedding with community information. We theoretically
demonstrate that such hybrid features can improve the performance of directed link prediction. To utilize such features efficiently, we also propose an approach to transform input graphs into directed line graphs so that nodes in the transformed graph can aggregate more information during graph convolutions. 
{Experiments on benchmark datasets show that our approach outperforms the state-of-the-art in most cases when 30\%, 40\%, 50\%, and 60\% of the connected links are used as training data, respectively}. 
\end{abstract}

}

\maketitle

\IEEEdisplaynontitleabstractindextext

\IEEEpeerreviewmaketitle

\IEEEraisesectionheading{\section{Introduction}\label{sec:introduction}}
Link prediction is essential in graph analysis, in which a model learns from graph data and makes predictions about whether there are connections between pairs of nodes \cite{butun2019pattern,wang2017predictive,kumar2019level,lee2021collaborative,guven2020applying}. It can have many applications, such as recommending friends in social networks \cite{gupta2021correcting,butun2018extension,liu2019link}, recommending products in commercial networks \cite{zhang2019domain,yin2022se}, analyzing interaction in chemical reaction networks \cite{yadati2020nhp}, recovering missing facts in graphs \cite{schlichtkrull2018modeling,chen2018pme} and reconstructing metabolic networks \cite{oyetunde2017boostgapfill,zhang2018link}, etc. As undirected graphs can be considered as a specific subset of directed graphs, many methods have been proposed to predict directed links.

{

Traditional approaches usually develop various indices with heuristic approaches to analyze the connection probability utilizing structural features. Zhou et al. improved link prediction accuracy by taking differences among neighboring nodes into account \cite{aziz2020link}. B{\"u}t{\"u}n et al. propose a pattern-based supervised approach to enhance link prediction accuracy in directed complex networks \cite{butun2019pattern}. Zhang et al. propose to combine preferential attachment indices to estimate connection likelihood between two nodes \cite{zeng2016link}. Chen et al. propose an enhanced local path link prediction similarity index based on the degree of the nodes at both ends of the path \cite{chen2021link}. Zheng et al. propose simrank-based indexing to predict links in large graphs \cite{zheng2013efficient}. Guo et al. introduce a novel method to address link prediction in directed networks by enhancing the Sorensen index and incorporating topological nearest-neighbors similarity, effectively utilizing network node information \cite{guo2023link}. Ghorbanzadeh et al. introduce a novel link prediction method based on common neighborhoods, addressing the issues of neglecting nodes without shared neighbors and connection direction in existing methods \cite{ghorbanzadeh2021hybrid}. As nodes in a community are more densely connected \cite{torri2023financial,berahmand2022graph} than those between different communities, some traditional approaches also adopt such information \cite{shang2020novel,ghasemian2020stacking,newman2018network}. For example, some traditional approaches propose using it to help infer connection probabilities in undirected networks \cite{PhysRevE.85.056112}, long-circle-like networks \cite{shang2022link}, and homology networks \cite{peixoto2022disentangling}.

}

{
Deep learning approaches can leverage deep neural networks to obtain more useful node representations automatically \cite{tong2021directed,zhang2018link,wu2021directed}. For undirected graphs, Zhang et al. unify a number of traditional approaches into a decaying framework using graph neural networks and predict undirected links \cite{zhang2018link}. Wu et al. extend Graph Convolutional Networks (GCNs) using social ranking theory, effectively addressing directed link prediction by accurately propagating directional information \cite{wu2021directed}. Wang et al. propose a weighted symmetric graph embedding approach for undirected link prediction \cite{wang2022weighted}. Ayoub et al. introduce a graph-based link prediction method using node similarity measures and path depth for undirected link prediction \cite{ayoub2020accurate}. Cai et al. also proposed undirected line graphs to help the analysis \cite{cai2021line}, in which each link in the original undirected graph is converted into a node, and the converted nodes are connected if their corresponding links in the original graph share a common endpoint \cite{choudhary2021atomistic,zhu2019relation,chen2017supervised}.

{Non-deep learning approaches are hindered by complex feature engineering and weak generalization capability, especially in complex or large-scale networks. Thus, introducing deep learning methods for link prediction in directed graphs is essential.} Kipf et al. propose a method to aggregate and propagate the directional information across layers of a GCN model to predict directed links \cite{kipf2016variational}. Tong et al. design a directed graph data augmentation method called Laplacian perturbation to improve the node represent performance \cite{tong2021directed}. Wu et al. integrate Graph Generative Adversarial Networks with Graph Contrastive Learning to predict directed links\cite{wu2023graph}. Yi et al. propose a novel autoencoder model based on high-order structures, integrating motif adjacency matrix learning and an attention scheme, to enhance link prediction accuracy in directed networks \cite{yi2022link}.

}

{However, when predicting links, approaches that rely solely on local information can be vulnerable to local data perturbations, while those using only global information may overlook important local features. To enhance information utilization from both local and global scopes, we propose to combine community information and line graph features, which have rarely been explored in deep learning-based directed link prediction approaches. }

Our contributions are as follows:
\begin{enumerate}
	\item {For directed graphs, we propose a novel graph neural network with hybrid-feature fusion that takes community information into account { } and proves its effectiveness theoretically.
	}
	
	\item 
We proposed a novel feature fusion approach that is able to transform directed graphs into directed line graphs, which can aggregate more information when predicting target links.

	\item Experiments with different datasets and different training sample proportions demonstrate that our proposed model achieves better results and converges with fewer epochs than other state-of-the-art models.
\end{enumerate}

\section{Related Work}
\subsection{Link Prediction}
There are many methods to predict links in the graph. These methods are mainly divided into three categories: heuristic methods, node embedding methods, and deep learning methods.

Heuristic approaches mainly rely on the neighboring nodes near the target node to calculate the similarity scores \cite{lu2015toward,yin2017evidential}.
The higher scores are, the more likely the nodes are linked to each other. The classical algorithms include Katz \cite{katz1953new}, rooted PageRank \cite{brin2012reprint}, and SimRank \cite{jeh2002simrank}. The performance of these algorithms highly relies on the number of neighborhood nodes.

Node embedding is inspired by natural language processing \cite{zhang2021adaptive}. The node sequences are generated by the skip-gram method, random walks, and other methods, so the information of the nodes in the graph can be learned to make predictions about the links. The classical algorithms include deepwalk \cite{perozzi2014deepwalk}, LINE \cite{tang2015line}, and node2vec \cite{qiu2018network}, which require relatively dense node connections. 

More and more methods based on deep learning have been proposed recently. Deep learning can learn node features, graph structure, and other information to predict links through neural networks. {The classical methods include SEAL \cite{zhang2018link}, DGGAN \cite{zhu2021adversarial}, DiGAE \cite{kollias2022directed}, LGLP \cite{cai2021line}, a pattern-based method \cite{butun2019pattern}, Gravity Graph VAE, and Gravity Graph AE \cite{salha2019gravity}.} However, these methods cannot take into account both local and global information. In this work, we propose to extract the target link and its neighbors to form a directed subgraph and learn local information. Supplement the global information with community information and node embedding, then transform a directed subgraph into a directed line graph to learn link features and link-to-link relationships.

\subsection{Graph Convolutional Network}
With the development of the graph neural network (GNN), learning graphical features using the graphical neural network has become an effective solution for a variety of graphical analysis tasks \cite{yuan2020structpool,ying2018hierarchical,kipf2016semi,li2024guest,bai2024haqjsk,li2024educross,li2024multimodal}. The graph convolutional network is an important part of the graph neural network and has an excellent performance in graph analysis. GCN is developed on the graph Fourier transform. The development of GCNs is mainly divided into two aspects. One is spectral-based methods, such as GCNs \cite{kipf2016semi} and ChebNet \cite{defferrard2016convolutional}, these kinds of methods have a solid foundation in spectral graph theory. There is also a spatial-based method, such as Graph-SAGE \cite{hamilton2017inductive}, MoNet \cite{monti2017geometric}, and MPNN \cite{gilmer2017neural}, which can be seen as the propagation of node information on the graph, and this type of method is simpler to understand. 

In this work, we propose to transform links in directed graphs into nodes in directed line graphs to improve the effectiveness of feature aggregation in GCNs.

\subsection{Node Embedding Learning}
Contrastive node embedding learning can learn representations from the perspective of positive and negative sample graphs \cite{ you2020graph,you2021graph}. Earlier approaches can suffer from two main problems. First, the structure of the input graph is changed in data augmentation for graph data \cite{you2020graph, zhu2021graph}, misleading the information transfer scheme in the subsequent comparative perspective learning. This change of the graph structure changes the graph structure information, especially for directional graphs, which have a greater impact due to the direction of their links. The second problem is that graph contrastive learning uses parameters defined in advance \cite{tian2020makes,xiao2020should}, which often does not fully utilize the contrastive information obtained from graph data augmentation and may end up with incomplete information acquisition. For the first problem, Directed Graph Contrastive Learning (DiGCL) \cite{tong2021directed} designs a Laplacian perturbation, and the perturbation can provide contrastive information without changing the structure of directed graphs. For the second problem, a contrastive learning framework for directed graphs is proposed, which automatically selects the appropriate contrastive views and is trained using multi-task course learning \cite{pentina2015curriculum,sarafianos2018curriculum}.

\section{METHODOLOGY}
\subsection{Problem Formulation}
{
In graph analysis, a directed graph can be represented as ${G=(V, E)}$, where ${V = \{v_i, v_j, ..., v_n}\}$ denotes the set of nodes, and ${E \subseteq V \times V}$ denotes the set of directed links between nodes. $\hat{d}$ is the average degree for a node. ${A}$ is the adjacency matrix of the directed graph. If there exists a directed link pointing from node ${v_i}$ to node ${v_j}$, then ${A_{i,j} = 1}$, and otherwise, ${A_{i,j} = 0}$. Without loss of generality, we formulate the problem to determine whether there is a link from node $v_x$ to node $v_y$, and we name such a link as the target link and the two corresponding nodes as the target nodes. We define the starting node of a link as the source node and the ending node of a link as the destination node.
 }
\subsection{Overall Framework}
{In this work, we propose Feature Fused Directed-line-graph model (FFD) to integrate local labeling information and global embedding information to predict links, and we also propose a novel node fusion method.
The overall approach for directed graphs is illustrated in Fig. \ref{main}.
Given a directed graph, we first obtain path-based node labels, community information, and node embeddings with different scopes, and then aggregate these features.
We transform the directed graph with aggregated features into a directed line graph. Finally, we perform graph convolutions on the directed line graph and generate node features to determine the existence of links between target nodes.
}

\begin{figure*}[tbp]
	\centering
	\includegraphics[width=7in,height=3in]{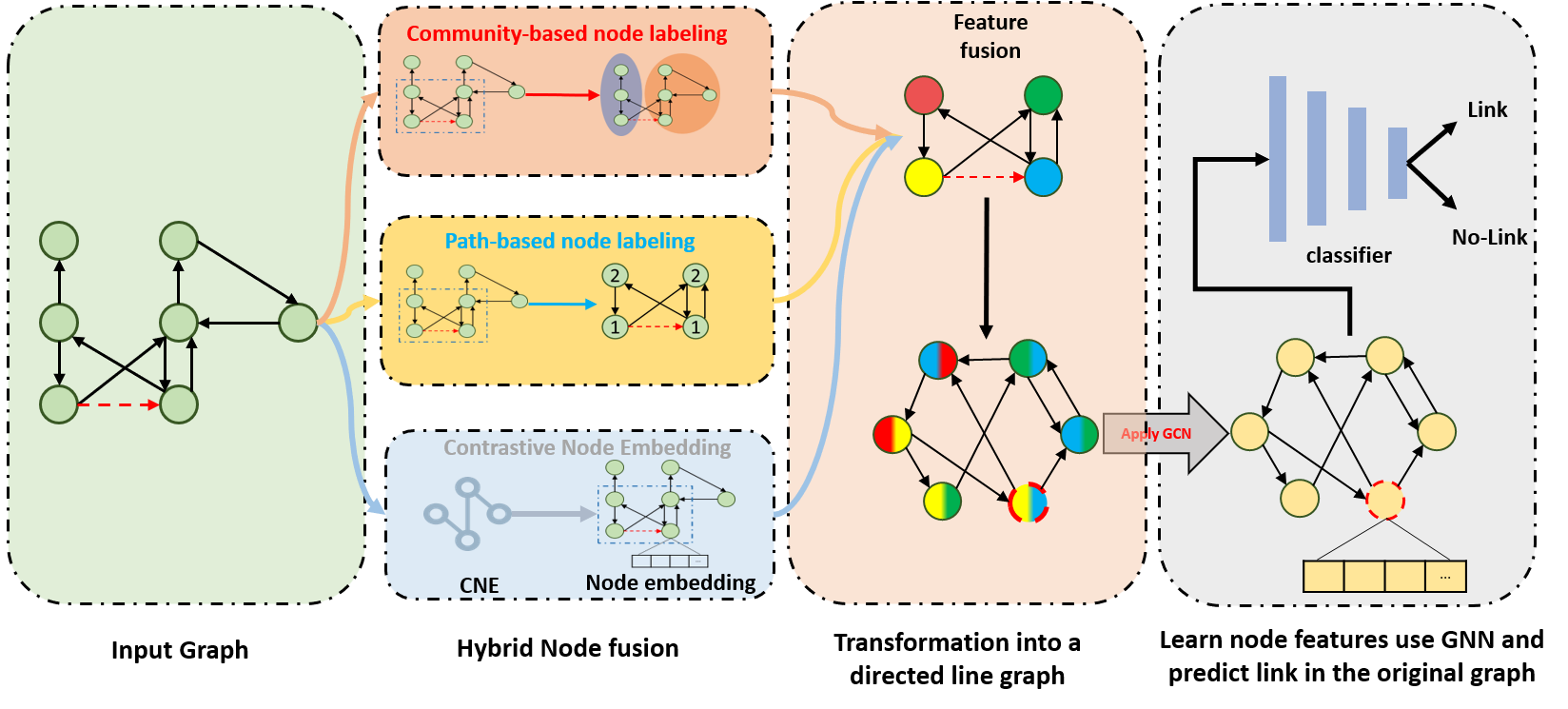}
	\DeclareGraphicsExtensions.
	\caption{An illustration of the proposed FFD. The target link is marked as a directed red dash line, and the corresponding target node in the directed line graph is marked as red dash circles. The main steps are as follows: (1) Obtain the path-based label information of the node, the community information of the node, and the node embedding. (2) Combine different node information and transform a directed subgraph graph into a directed line graph. (3) Propagate features using graph convolution networks and predicts the target link with the corresponding node feature from the directed line graph.}
	\label{main}	 
\end{figure*}

\subsection{Feature Fusion Directed line graph modeling}

\subsubsection{Hybrid Node fusion}

To obtain local information on target links, we propose a hybrid node fusion method. In addition, besides existing labeling methods, we propose a community-based node labeling method and demonstrate theoretically that it can help improve prediction accuracy when integrating with others.

\textbf{Path-based Node Labeling}

Path-based node labeling methods focus on the neighborhood information of target nodes, and they have been adopted in link prediction tasks.
Such labeling methods typically rely on subgraphs generated by a few steps of hopping from target nodes \cite{cai2021line}, as computing paths from the whole graph can be very expensive.

Without loss of generality, we consider a path-based node labeling algorithm $\zeta _{P}$ as follows
\begin{align}
O_P=\zeta _{P}(A)
\label{labelA},
\end{align}
where the input $A$ is an adjacency matrix, and each row of the output $O_P\in \{0,1\}^{N\times D_P}$ indicates the path-based labels of a corresponding node. $D_P$ is the size of output features.

\textbf{Community-based Node Labeling} 

{Community structures can reflect the probability of node connections and can help nodes to learn features from a medium scope, and many existing approaches \cite{PhysRevE.70.066111,zhao2021community} can be used.
We define a community detection algorithm $\zeta _{C}$ as follows}
\begin{align}
	O_C=&\zeta _{C}(A),\label{labelC}
\end{align}
where the input $A$ is an adjacency matrix, and each row of the output $O_C\in \{0,1\}^{N\times D_C}$ indicates the community affiliation of a node. {$D_C$ is the number of communities. }

\textbf{Contrastive Node Embedding}

{Contrastive node embedding takes all the nodes into account, aiming to learn representations of nodes such that similar nodes are brought closer in the embedding space, while dissimilar nodes are pushed apart \cite{you2020graph}. Contrastive node embedding can help nodes to learn features from a much larger scope, from the whole graph, and many existing approaches \cite{tong2021directed,zhang2022costa} can be used. Without loss of generality, we define contrastive node embedding $\zeta _{E}$ as follows}
\begin{align}
	O_E=\zeta _{E}(A)
	\label{labelE},
\end{align}
where the output is $O_E \in \mathbb{R}^{N\times D_E}$ and {$D_E$} is the size of embedded features.

{
\textbf{Hybrid Node fusion}

Path-based node labels, community-based node labels, and contrastive embedding features can be considered complementary features from different scopes.
To enhance the prediction performance, we can combine these kinds of features and feed them into the next step
\begin{align}
	O_H=Combination(O_P,O_C, O_E)\label{label},
\end{align}
{where $Combination(\cdot)$ denotes feature combination function, $O_H \in \mathbb{R}^{N\times D_H}$ with $D_H$ as the size of hybrid features. $O_H[i]$ denotes the hybrid node fusion of node $v_i$. By default, we use concatenation function as the combination function, as many other deep learning approaches do \cite{hamilton2017inductive,vaswani2017attention,velickovic2017graph,xu2018representation}.
}}

\textbf{Effectiveness of Community-based Node Labels}

{As community-based node labeling has rarely been used in directed link prediction problems with graph neural networks, in this work, we theoretically demonstrate its effectiveness.
For ease of analysis, we consider that input networks are generated from a stochastic blockmodel (SBM) \cite{tang2022asymptotically} defined as follows. Many studies provide theoretical proof based on SBM and apply it to different real-world datasets\cite{tang2022asymptotically,braun2022iterative,karrer2011stochastic}. We also provide theoretical proof based on SBM.}

\begin{definition}
In the adopted SBM, there are total $K$ communities of the same size. The connection probability between nodes in the same community is $p$ and that between nodes from different communities is $q$.
\end{definition}
{
\begin{definition}
A non-community-based link prediction algorithm is defined as as $\zeta _{\overline{C}}$.
\end{definition}
}
{
We consider a procedure $f$ that can map outputs (e.g, $O_P$, $O_C$, $O_E$, or their combinations) from different algorithms into link prediction results, given target nodes 
$v_x$ and $v_y$. If a link is predicted, the output is greater than zero, and otherwise, the output is less than zero. With respect to an specific algorithm, for example, $\zeta _{C}$, we abbreviate $f(\zeta_{C}(A),v_x,v_y)$ with $f(\zeta_{C})$. 
With a non-community-based labeling algorithm, we use the following expression to represent the prediction with hybrid node fusion:
\begin{align}
	W_{\overline{c}}f(\zeta _{\overline{C}})+W_c f(\zeta _{C}),
	\end{align} 
where $W_{\overline{c}}$ and $W_c$ are learnable weights. 

}

{We use $\mathbb{E}$ to represent the probability of correct predictions given ground-truth community labels. Thus, the expected probability for non-community-based approaches can be defined as follows}{, the expected prediction probability of $\zeta _{\overline{C}}$ over different cases is defined follows} 
\begin{align}
	\begin{split}
		\mathbb{E}(f(\zeta _{\overline{C}}))
		=&\left \{
		\begin{array}{ll}
			E_{\overline{C},U},                   & \text{if $\ f(\zeta _{\overline{C}})<0$ and $A_{x,y}=0$,} \\
			1-E_{\overline{C},U},    & \text{if $\ f(\zeta _{\overline{C}})>0$ and $A_{x,y}=0$,} \\
			1-E_{\overline{C},L},                                & \text{if $\ f(\zeta _{\overline{C}})<0$ and $A_{x,y}=1$,} \\
			E_{\overline{C},L},    & \text{if $\ f(\zeta _{\overline{C}})>0$ and $A_{x,y}=1$.} \\
		\end{array}
		\right.
	\end{split},
\end{align}
where for a linked pair of nodes, the expected prediction of being linked is $E_{\overline{C},L}$, and that of being unlinked is $1-E_{\overline{C},L}$. For an unlinked pair of nodes, the expected prediction of being unlinked is $E_{\overline{C},U}$, and that of being linked is $1-E_{\overline{C},U}$.

{
\begin{definition}
Given a community detection algorithm $\zeta _{C}$ that can decide whether a pair of nodes are from the same community or not, we define its error rate as $\varepsilon _{0}$ for a pair of disconnected nodes and that as $\varepsilon _{1}$ for pairs of connected nodes. 
\end{definition}
}
\begin{myTheo}
{
	According to Definitions 1, 2, and 3, given an adjacency matrix $A$ that is generated by the above stochastic blockmodel, if the following condition is satisfied
	\begin{align}
	&(p-q)(\frac{p}{p+(K-1)q}+\frac{{1-p}}{K - [p+(K-1)q]})\nonumber\\
&+1-\varepsilon _{0}-\varepsilon _{1} \nonumber\\
	\ge &\frac{1-W_{\overline{c}}}{W_c} (E_{\overline{C},U}+E_{\overline{C},L})\label{eq7}.
\end{align}}
{ 
The probability of making correct predictions using hybrid node fusion
	$W_{\overline{c}}f(\zeta _{\overline{C}})+W_cf(\zeta _{C})$  
	is expected to be higher than that of using non-community-based labeling $f(\zeta _{\overline{C}})$ only. 
	\label{T0}}
\end{myTheo}

\begin{proof}

Typically, benchmark data sets for link prediction methods contain linked and unlinked pairs of nodes with the same proportion. By random selection, the probability that a pair of linked nodes in the benchmark set is from the same community is 
\begin{align}
	\theta_L=&\frac{(\frac{1}{K})^2Kp}{(\frac{1}{K})^2Kp+[1-(\frac{1}{K})^2K]q}\\
	=&\frac{p}{p+(K-1)q},
\end{align}
and similarly, that 
from two different communities is $1-\theta_L$.

Likewise, the probability that an unlinked pair of nodes are from the same community as $\theta_U$
\begin{align}
	\theta_U=&\frac{(\frac{1}{K})^2K(1-p)}{(\frac{1}{K})^2K(1-p)+[1-(\frac{1}{K})^2K](1-q)}\\
	=&\frac{{1-p}}{K - [p+(K-1)q]},
\end{align}
and that from different communities is $1-\theta_U$.

a's

As $\zeta _{C}$ can identify the block structure correctly, the predicted link probability can be approximated by the generative parameters $p$ and $q$: 
{
\begin{align}
\begin{split}
\mathbb{E}(\zeta _{C})
= \left \{
\begin{array}{ll}
  \theta_U (1-p)+(1-\theta_U) (1-q)-\varepsilon _{0} ,   \\
 \ \ \ \ \ \ \ \ \ \  \text{if $\ f(\zeta _{C})<0$ and $A_{x,y}=0\ $}, \\
  1-[\theta_U (1-p)+(1-\theta_U) (1-q)]+\varepsilon _{0}, \\ \ \ \ \ \ \ \ \ \ \    \text{if $\ f(\zeta _{C})>0$ and $A_{x,y}=0\ $}, \\
    1-[ \theta_L p+(1-\theta_L)q]+\varepsilon _{1},\\
 \ \ \ \ \ \ \ \ \ \   \text{if $\ f(\zeta _{C})<0$ and $A_{x,y}=1\ $}, \\
    \theta_L p+(1-\theta_L)q-\varepsilon _{1},  \\
 \ \ \ \ \ \ \ \ \ \  \text{if $\ f(\zeta _{C})>0$ and $A_{x,y}=1\ $}. \\
\end{array}
\right.
\end{split}\label{EC}
\end{align}
}

where given node $x$ and $y$ are connected, $\theta_L p$ denotes the probability that two nodes are in the same community and the link is correctly predicted, and $(1-\theta_L)q$ denotes the probability that two nodes are in different communities with correct prediction. Similarly, given node $x$ and $y$ are disconnected, $\theta_U (1-p)$ denotes the probability that two nodes are in the same community and the disconnection is correctly predicted, and $(1-\theta_U)(1-q)$ denotes the probability that two nodes are in different communities with correct prediction.

Consequently, for the hybrid node fusion, the expected probability is 
{
\begin{align}
&\mathbb{E}(W_{\overline{c}}f(\zeta _{\overline{C}})+W_cf(\zeta _{C}))=\nonumber\\
&\left \{
\begin{array}{ll}
    W_{\overline{c}}E_{\overline{C},U}+  W_c[\theta_U (1-p)+(1-\theta_U) (1-q)-\varepsilon _{0}],\\ \ \ \   \text{if}\ W_{\overline{c}}f(\zeta _{\overline{C}}) +W_cf(\zeta _{C})<0\ \text{and} \ A_{x,y}=0,\\           
 W_{\overline{c}}(1-E_{\overline{C},U})+W_c\{ 1-[\theta_U (1-p)\\
 \ \ \ \ \  \ \ \ \ \  +(1-\theta_U)(1-q)-\varepsilon _{0}]\}, \\ \ \ \   \text{if}\ W_{\overline{c}}f(\zeta _{\overline{C}})+W_cf(\zeta _{C})>0 \text{ and } \ A_{x,y}=0,  \\ 
      W_{\overline{c}}(1-E_{\overline{C},L})+ W_c\{  1-[\theta_L  p +(1-\theta_L)q-\varepsilon _{1}]\},   \\  \ \ \     \text{if}\ W_{\overline{c}}f(\zeta _{\overline{C}})+W_cf(\zeta _{C})<0 \text{ and } A_{x,y}=1 ,\\ 
    W_{\overline{c}}E_{\overline{C},L}+ W_c[\theta_L  p +(1-\theta_L)q-\varepsilon _{1}],\\
\ \ \   \text{if } W_{\overline{c}}f(\zeta _{\overline{C}})+W_cf(\zeta _{C})>0 \text{ and } A_{x,y}=1.
\end{array}
\right.\label{ENC}
\end{align}
	}

{The expectation of the sum of random variables is equal to the sum of their respective expectations. Thus, the probability of predicting the link correctly with hybrid fusion is as follows:}

{
\begin{align}
&W_{\overline{c}}E_{\overline{C},U}+ W_c[\theta_U (1-p) +(1-\theta_U) (1-q)-\varepsilon _{0}]\nonumber\\
 & \ \  +W_{\overline{c}}E_{\overline{C},L}+ W_c[\theta_L  p +(1-\theta_L)q-\varepsilon _{1}] \nonumber\\
=&W_{\overline{c}}(E_{\overline{C},U}+E_{\overline{C},L})+W_c[\theta_L(p-q)+\theta_U(q-p)\nonumber\\
&\ \  +1-\varepsilon _{0}-\varepsilon _{1}]\label{eq14}\\
=&W_{\overline{c}}(E_{\overline{C},U}+E_{\overline{C},L})+W_c[(p-q)(\frac{p}{p+(K-1)q}\nonumber\\
&\ \  +\frac{{1-p}}{K - [p+(K-1)q]})+1-\varepsilon _{0}-\varepsilon _{1}]\label{eq15}\\
\ge& W_{\overline{c}}(E_{\overline{C},U}+E_{\overline{C},L})+
{(1-W_{\overline{c}})(E_{\overline{C},U}+E_{\overline{C},L})}\\
\ge& E_{\overline{C},U}+E_{\overline{C},L},\label{eq17}
 \end{align}}
where Eq.($\ref{eq15}$) is obtained by substituting $\theta_U$ and $\theta_L$ with Eq.($\ref{eq14}$), and Eq.($\ref{eq17}$) is obtained given that the condition of Eq.($\ref{eq7}$) is satisfied.
\end{proof}

 { From this theorem, we can find that, when a community detection algorithm of $\zeta _{C}$ is of relatively high accuracy, namely, with a small error rate $\varepsilon _{0}$ and $\varepsilon _{1}$, the condition of Eq. (\ref{eq7}) is likely to be satisfied, and vice versa.

To qualitatively analyze the impact of this theorem,
We define the left part of Eq.(\ref{eq7})
as $g(p,q, K)$, and for the ease of analysis, we assume that the error rates of the community detection algorithm are zero. Thus, when $p\in (0,1]$ and $q\in (0,1]$, the partial derivative of $g(p,q,K)$ with respect to $p$ is}
\begin{align}
&\frac{\partial g(p,q,K)}{\partial p}\nonumber\\
=&\dfrac{p}{p+\left(K-1\right)\,q}+\left(p-q\right)\,\nonumber\\
&\left(\dfrac{\left(K-1\right)\,q}{\left({p+\left(K-1\right)\,q}\right)^{2}}+\dfrac{\left(K-1\right)\,q-K+1}{\left({-p-\left(K-1\right)\,q+K}\right)^{2}}\right)\nonumber\\
&+\dfrac{1-p}{-p-\left(K-1\right)\,q+K}\\
\ge& 0,
\end{align}
and that with respect to $q$ is
\begin{align}
&\frac{\partial g(p,q,K)}{\partial q}\nonumber\\
=&-\dfrac{p}{\left(K-1\right)\,q+p}+\left(p-q\right)\,\nonumber\\
&\left(\dfrac{\left(1-K\right)\,\left(p-1\right)}{\left({-\left(K-1\right)\,q-p+K}\right)^{2}}-\dfrac{\left(K-1\right)\,p}{\left({\left(K-1\right)\,q+p}\right)^{2}}\right)\nonumber\\
&-\dfrac{1-p}{-\left(K-1\right)\,q-p+K}\\
\le& 0.
\end{align}
Therefore, when $p\in (0,1]$ and $q\in (0,1]$, the left side of Eq. (\ref{eq7}) increases as $p$ increases and $q$ decreases. 
Consequently, the condition of Eq. (\ref{eq7}) in Theorem \ref{T0} is easier to be satisfied with larger $p$ and smaller $q$, namely, a more obvious community structure. This is in consistency with the intuition that a clearer community structure can be more helpful for link predictions. 

\subsubsection{Node-Link Transformation based on Directed Graph}

For undirected graphs, transforming a subgraph including target nodes into a line graph can improve the efficiency of GNN in learning link features \cite{cai2021line}. 

For directed graphs, in this work, we propose a corresponding transformation.
In this transformation, each link in an input graph has been transformed into a node 
and a link prediction problem is transformed into a node prediction one.

We use the following function to represent transformation from $G$ to 
$\mathscr{G}$ :
\begin{equation}
	TR(G)=\mathscr{G} =\{\mathscr{V}, \mathscr{E}\}, \label{TR}
\end{equation}
where ${TR(\cdot)}$ denotes the transformation function, $\mathscr{G}$ denotes the directed line graph corresponding to $G$. ${\mathscr{V}}$ denotes the nodes in the directed line graph, and ${\mathscr{E}}$ denotes the links in the directed line graph. 
The relationship between these variables can be as follows:
\begin{align}
	\mathscr{V} = \{(v_i, v_j)&\ |\ (v_i,v_j)\in E\}, \label{TRv}\\
	\mathscr{E} = \{[(v_i,v_j)&,(v_j,v_k)]\ |(v_i,v_j),(v_j,v_k)\in E\}. \label{TRe}
\end{align}
It means that, when $(v_i,v_j)$ is an link in the input graph, it becomes a corresponding node in the line graph. When there is a directed link from $v_i$ to $v_j$ and from $v_j$ to $v_k$ in the input graph, there is a corresponding link from node $(v_i,v_j)$ to node $(v_j,v_k)$ in the line graph. 
An example of transforming a directed graph 
into a directed line graph is shown in Fig. \ref{line graph}. 
The feature of each node, for example, $(v_i,v_j)$ in the line graph, is obtained by concatenating corresponding node features, $[O_H[i], O_H[j]]$, in the input graph.
{If a subgraph has $\hat{n} $ nodes with average degree $\hat{d} $, the corresponding number of nodes in the transformed directed line graph is 
$\hat{n} \hat{d} $. Thus, for each graph convolution operation, the transformed input graph is $\hat{d}$ times larger than the original input graph.

}
\begin{figure}[!t]
	\centering
	\includegraphics[height=1in]{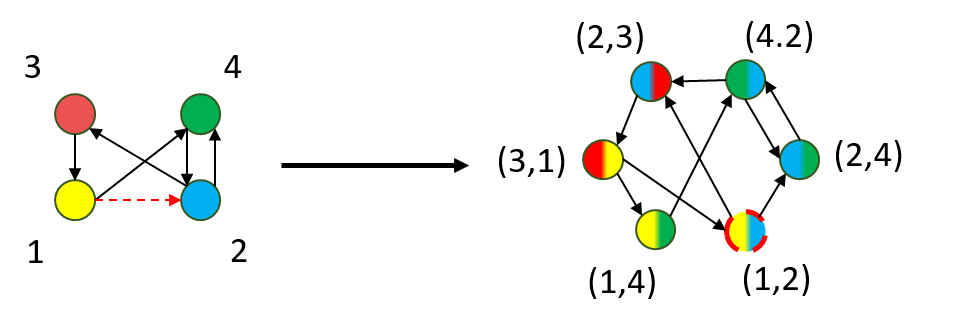}
	\DeclareGraphicsExtensions.
	\caption{Illustration of a directed graph transformed into a directed line graph.
Each node in a directed line graph corresponds to a unique directed link in an input graph. The colors on the left and right of the nodes in the right graph represent features representing the source node and destination node of the corresponding link in the left graph. Predicting the node $(1, 2)$ in the right graph is equivalently predicting the link from node one to node two in the left graph.}
	\label{line graph}
\end{figure}

\subsubsection{Link Prediction}
{In this part, we employ GNN in the directed line graph to learn features of nodes, which is equivalent to learning features of links in an original directed graph. Thus, the prediction of directed links is transformed into the prediction of nodes.

The GNN operation is performed on the directed line graph with the following equation:
\begin{align}
	H_{l+1}=GNN(H_l,\mathscr{G}),\label{LGCN}
\end{align}
and a fully connected network is applied on the final output $H_L$, to form a classifier
\begin{align}
	B=&FCN(H_{L}),
\label{bfcnhl}
\end{align}
where ${B}$ is a vector with each element between $0$ and $1$ representing the existence likelihood of the corresponding link in the original graph, $FCN$ denotes a fully connected output network, and its last layer is a logistic layer.

Therefore, the loss function is as follows}
{
\begin{align}
	\mathcal{L}_{CE}=&\frac{1}{N} \sum_{\mathcal{v}_k\in \mathscr{V}}-[R_k\cdot \log(B_k)+(1-R_k)\cdot \log(1-B_k)],\label{loss}
\end{align}
}
{where $\mathcal{L}_{CE}$ is cross-entropy loss, $R_k$ denotes the true label of node $\mathcal{v}_k$ in the transformed directed graph, corresponds to the true label of target links in the original graph.}

{
\subsection{Pseudocode}
The pseudo-code of the proposed approach for the inference phase is depicted in Algorithm \ref{Algorithm}. A path-based node labeling represented by Eq.(\ref{labelA}) is usually a predefined hash function.
A community-based node labeling algorithm represented by Eq.(\ref{labelC}) is an unsupervised approach.
A node embedding approach represented by Eq.(\ref{labelE}) is typically a self-supervised approach. With these features, a GNN represented by Eq.(\ref{LGCN}) can be trained given ground-truth link existence labels.
}

\IncMargin{1em}
\begin{algorithm} \SetKwData{Left}{left}\SetKwData{This}{this}\SetKwData{Up}{up} \SetKwFunction{Union}{Union}\SetKwFunction{FindCompress}{FindCompress} \SetKwInOut{Input}{input}\SetKwInOut{Output}{output}
	
	\Input{the directed graph ${G}$, \\
the target link from node ${v_x}$ to node ${v_y}$
} 
	\Output{Prediction result}
	 \BlankLine 
	 
	 \emph{Obtain the path-based node labels by using Eq.($\ref{labelA}$)}\; 
	\emph{Obtain the community-based node labels by using Eq.($\ref{labelC}$)}\; 
	\emph{Obtain the contrastive node embedding by using Eq.($\ref{labelE}$)}\; 	
	 \emph{Obtain hybrid node fusion by using Eq.($\ref{label}$)}\; 
	\emph{Transform the subgraph around the target link into a line graph 
by using Eq. ($\ref{TR}$) and generate node features of the line graph 
}\; 
	\emph{Apply GNNs according to Eq.(${\ref{LGCN}}$) to extract final features in the subgraph and predict link existences by Eq. (\ref{bfcnhl}) according to corresponding node features}\;
 	 	  \caption{Proposed Approach for Directed Link Prediction}
 	 	  \label{Algorithm} 
 	 \end{algorithm}
 \DecMargin{1em}

{

\begin{table*}[!t]

	\renewcommand{\arraystretch}{1}
	\caption{Data info}
	\label{dataset}
	\centering
	\begin{tabular}{c| c c c c c c}
		\toprule
		dataset &Cora\_ml & Cora & Citeseer& {Pubmed}& {Bitcoin}& {p2p-Gnutella04}\\
		\midrule
		Number of nodes &2993&2707&3308&{19717}&{5881}&{10876}\\
		\hline
		Number of links &8417&5430&4716&{44338}&{35592}&{39994}\\
		\hline
		Number of binary-directional links between nodes &516&302&234&{10}&{7050}&{0}\\
		\hline
		Proportion of binary-directional links &0.0610&0.0550&0.0490&{0.0002}&{0.1980}&{0}\\
		\bottomrule
	\end{tabular}

\end{table*}
}
\begin{table*}[!t]
	\renewcommand{\arraystretch}{1}
	\caption{AUC and AP results compared to baseline methods (60Tr).}
	\label{result60}
	\centering
	\begin{tabular}{l| c c c c c c c c c c c c}
		
		\toprule
		Model & \multicolumn{2}{|c|}{Cora\_ml} & \multicolumn{2}{|c|}{Cora} & \multicolumn{2}{|c|}{Citeseer}&\multicolumn{2}{|c|}{Pubmed}&\multicolumn{2}{|c}{Bitcoin}&\multicolumn{2}{|c}{p2p-Gnutella04}\\
		\midrule
		Metric&AUC&AP&AUC&AP&AUC&AP&AUC&AP&AUC&AP&AUC&AP\\
		\hline
		GAE&0.76&0.77&0.67&0.68&0.65&0.65&0.75&0.77&0.87&0.88&0.52&0.57\\
		\hline
		GVAE&0.80&0.84&0.71&0.77&0.68&0.75&\underline{0.78}&\underline{0.80}&0.87&\underline{0.90}&0.54&0.58\\
		\hline
		STGAE&0.75&0.76&0.76&0.79&0.60&0.62&0.70&0.77&0.90&\textbf{0.91}&{0.67}&{0.64}\\
		\hline
		STGVAE&0.80&0.84&0.72&0.77&0.67&0.74&0.72&0.77&\underline{0.91}&\underline{0.90}&\underline{0.73}&\underline{0.72}\\
		\hline
		GGAE&\underline{0.85}&\underline{0.88}&0.80&\underline{0.85}&0.71&\underline{0.78}&0.77&\underline{0.80}&0.86&0.89&{0.68}&{0.70}\\
		\hline
		GGVAE&\underline{0.85}&0.84&\underline{0.82}&0.80&\underline{0.76}&0.73&0.71&0.70&0.87&0.86&{0.67}&{0.65}\\
		\hline
		\textbf{FFD}&\textbf{0.88}&\textbf{0.89}&\textbf{0.85}&\textbf{0.86}&\textbf{0.81}&\textbf{0.83}&\textbf{0.83}&\textbf{0.82}&\textbf{0.92}&\underline{0.90}&\textbf{0.81}&\textbf{0.74}\\
		\bottomrule
		
	\end{tabular}
\end{table*}

\begin{table*}[!t]
	\renewcommand{\arraystretch}{1}
	\caption{AUC and AP results compared to baseline methods (50Tr).}
	\label{result50}
	\centering
	\begin{tabular}{l| c c c c c c c c c c c c}
		
		\toprule
		Model & \multicolumn{2}{|c|}{Cora\_ml} & \multicolumn{2}{|c|}{Cora} & \multicolumn{2}{|c|}{Citeseer}&\multicolumn{2}{|c|}{Pubmed}&\multicolumn{2}{|c}{Bitcoin}&\multicolumn{2}{|c}{p2p-Gnutella04}\\
		\midrule
		Metric&AUC&AP&AUC&AP&AUC&AP&AUC&AP&AUC&AP&AUC&AP\\
		\hline
		GAE&0.73&0.74&0.66&0.68&0.64&0.64&0.76&0.74&0.83&0.88&0.53&0.57\\
		\hline
		GVAE&0.79&\underline{0.84}&0.64&0.71&\underline{0.72}&\underline{0.76}&0.75&0.77&0.84&0.88&0.53&0.57\\
		\hline
		STGAE&0.71&0.72&0.61&0.62&0.63&0.67&\underline{0.77}&\underline{0.79}&\underline{0.89}&\textbf{0.91}&0.57&0.58\\
		\hline
		STGVAE&\underline{0.79}&0.83&0.66&0.72&0.64&0.71&0.73&0.76&\underline{0.89}&\underline{0.90}&\underline{0.68}&0.68\\
		\hline
		GGAE&0.77&0.81&0.76&\underline{0.82}&0.68&\underline{0.76}&\underline{0.77}&\textbf{0.80}&0.85&0.89&\underline{0.68}&\underline{0.69}\\
		\hline
		GGVAE&0.78&0.77&\underline{0.82}&0.79&0.71&0.67&0.68&0.66&0.85&0.83&0.66&0.64\\
		\hline
		\textbf{FFD}&\textbf{0.86}&\textbf{0.87}&\textbf{0.83}&\textbf{0.84}&\textbf{0.80}&\textbf{0.81}&\textbf{0.81}&\textbf{0.80}&\textbf{0.92}&\underline{0.90}&\textbf{0.80}&\textbf{0.74}\\
		\bottomrule
		
	\end{tabular}
\end{table*}

\begin{table*}[!t]
	\renewcommand{\arraystretch}{1}
	\caption{AUC and AP results compared to baseline methods (40Tr).}
	\label{result40}
	\centering
	\begin{tabular}{l| c c c c c c c c c c c c}
		
		\toprule
		Model & \multicolumn{2}{|c|}{Cora\_ml} & \multicolumn{2}{|c|}{Cora} & \multicolumn{2}{|c|}{Citeseer} & \multicolumn{2}{|c|}{Pubmed}& \multicolumn{2}{|c}{Bitcoin}&\multicolumn{2}{|c}{p2p-Gnutella04}\\
		\midrule
		Metric&AUC&AP&AUC&AP&AUC&AP&AUC&AP&AUC&AP&AUC&AP\\
		\hline
		GAE&\underline{0.74}&0.76&0.61&0.61&0.64&0.66&0.72&0.74&0.82&0.86&0.47&0.53\\
		\hline
		GVAE&0.72&0.77&0.64&0.70&0.59&0.65&0.70&0.72&0.81&0.86&0.52&0.56\\
		\hline
		STGAE&0.64&0.64&0.60&0.61&0.59&0.63&0.71&0.73&\underline{0.88}&\textbf{0.90}&0.47&0.48\\
		\hline
		STGVAE&0.72&0.78&0.61&0.64&0.61&0.67&0.68&0.72&0.86&\underline{0.89}&0.61&0.62\\
		\hline
		GGAE&\underline{0.74}&\underline{0.79}&0.71&\underline{0.77}&0.68&\underline{0.75}&\underline{0.73}&\underline{0.78}&0.85&0.88&\underline{0.68}&\underline{0.69}\\
		\hline
		GGVAE&0.73&0.69&\underline{0.78}&0.71&\underline{0.70}&0.68&0.63&0.62&0.81&0.75&0.64&0.61\\
		\hline
		\textbf{FFD}&\textbf{0.84}&\textbf{0.85}&\textbf{0.83}&\textbf{0.83}&\textbf{0.78}&\textbf{0.76}&\textbf{0.80}&\textbf{0.79}&\textbf{0.91}&\textbf{0.90}&\textbf{0.80}&\textbf{0.73}\\
		\bottomrule
		
	\end{tabular}
\end{table*}
\begin{table*}[!t]
	\renewcommand{\arraystretch}{1}
	\caption{AUC and AP results compared to baseline methods (30Tr).}
	\label{result30}
	\centering
	\begin{tabular}{l| c c c c c c c c c c c c}
		
		\toprule
		Model & \multicolumn{2}{|c|}{Cora\_ml} & \multicolumn{2}{|c|}{Cora} & \multicolumn{2}{|c|}{Citeseer} & \multicolumn{2}{|c|}{Pubmed} & \multicolumn{2}{|c}{Bitcoin}&\multicolumn{2}{|c}{p2p-Gnutella04}\\
		\midrule
		Metric&AUC&AP&AUC&AP&AUC&AP&AUC&AP&AUC&AP&AUC&AP\\
		\hline
		GAE&0.61&0.63&0.54&0.56&0.54&0.55&0.68&0.71&0.78&0.81&0.47&0.52\\
		\hline
		GVAE&0.64&0.71&0.55&0.61&0.56&0.61&0.68&0.70&0.76&0.83&0.51&0.54\\
		\hline
		STGAE&0.54&0.57&0.52&0.53&0.55&0.59&0.64&0.67&0.79&0.81&0.38&0.43\\
		\hline
		STGVAE&0.65&0.70&0.56&0.61&0.55&0.58&0.68&0.69&\underline{0.83}&\underline{0.87}&0.57&0.58\\
		\hline
		GGAE&0.70&\underline{0.76}&0.62&\underline{0.68}&0.58&\underline{0.64}&\underline{0.72}&\underline{0.76}&\underline{0.83}&0.86&\underline{0.67}&\underline{0.68}\\
		\hline
		GGVAE&\underline{0.71}&0.69&\underline{0.66}&0.60&\underline{0.60}&0.56&0.55&0.57&0.80&0.74&0.58&0.54\\
		\hline
		\textbf{FFD}&\textbf{0.82}&\textbf{0.83}&\textbf{0.81}&\textbf{0.81}&\textbf{0.74}&\textbf{0.70}&\textbf{0.78}&\textbf{0.76}&\textbf{0.90}&\textbf{0.89}&\textbf{0.79}&\textbf{0.73}\\
		\bottomrule
		
	\end{tabular}
\end{table*}

\begin{table*}[!t]
	\renewcommand{\arraystretch}{1}
	\caption{AUC results compared to traditional methods (50Tr).}
	\label{resulttraditional}
	\centering
	\begin{tabular}{l| c c c c c c }
		
		\toprule
		Model   &Cora\_ml&{Cora} & {Citeseer}  &pubmed& {Bitcoin}&p2p-Gnutella04\\
		\midrule
		CN&0.67 & 0.6&0.58 &0.57 &0.63&0.51  \\
		\hline
		Salton&0.67&0.6&0.58&0.57&0.62&0.51\\
		\hline
		Jaccard&0.65&0.57&0.47&0.52&0.61&0.49\\
		\hline
		Sorenson&0.65&0.57&0.47&0.52&0.60&0.49\\
		\hline
		HPI&0.42&0.31&0.18&0.21&0.41&0.31\\
		\hline
		HDI&0.65&0.57&0.47&0.52&0.61&0.49\\
		\hline
LHN-1&0.42&0.31&0.18&0.21&0.41&0.31\\
		\hline
PA&0.69&0.59&0.58&0.71&0.49&\underline{0.74}\\
		\hline
AA&0.67&0.60&0.58&0.57&0.63&0.51\\
		\hline
RA&0.67&0.60&0.58&0.57&0.63&0.51\\
		\hline
LP&0.78&0.67&0.64&0.76&0.49&0.66\\
		\hline
Katz&0.77&0.68&0.66&0.73&0.49&0.71\\
		\hline
ACT&0.62&0.54&0.47&0.72&0.44&0.66\\
		\hline
Cos&\underline{0.83}&\underline{0.72}&\underline{0.66}&\underline{0.79}&\underline{0.77}&0.49\\
		\hline
		\textbf{FFD}&\textbf{0.86}&\textbf{0.83}&\textbf{0.80}&\textbf{0.81}&\textbf{0.92}&\textbf{0.80}\\
		\bottomrule
		
	\end{tabular}
\end{table*}

\begin{table*}[!t]
	\renewcommand{\arraystretch}{1}
	\caption{{AUC and AP results compared to more methods (50Tr).}}
	\label{result50+}
	\centering
	\begin{tabular}{l| c c c c c c }
		
		\toprule
		Model  & \multicolumn{2}{|c|}{Cora} & \multicolumn{2}{|c|}{Citeseer} &  \multicolumn{2}{|c}{Bitcoin}\\
		\midrule
		Metric&AUC&AP&AUC&AP&AUC&AP\\
		\hline
		DiGAE&0.74&\underline{0.79}&0.78&\underline{0.81}&0.75&0.83\\
		\hline
  DiGAE\_single\_layer&0.71&0.78&\textbf{0.85}&\textbf{0.89}&0.75&\underline{0.84}\\
		\hline
  JPA&0.65&0.60&0.63&0.58&0.66&0.75\\
		\hline
  LINE&0.69&0.64&0.69&0.64&0.62&0.61\\
		\hline
  Node2vec&0.71&0.67&0.72&0.68&0.75&0.72\\
		\hline
  SEAL&0.71&0.77&0.72&0.74&0.75&0.78\\
		\hline
  ARVGA&0.70&0.77&0.68&0.74&0.78&0.75\\
		\hline
  GIC&0.62&0.66&0.66&0.65&0.80&0.77\\
		\hline
  LMA&0.69&0.76&0.67&0.73&0.82&0.81\\
		\hline
  MaskGAE&0.68&0.73&0.69&0.71&0.83&0.80\\
		\hline
  S2GAE-SAGE&\underline{0.79}&0.76&0.70&0.72&0.83&0.83\\
		\hline
  ClusterLP&0.73&0.78&0.73&0.76&\underline{0.85}&\underline{0.84}\\
		\hline
		
		\textbf{FFD}&\textbf{0.83}&\textbf{0.84}&\underline{0.80}&\underline{0.81}&\textbf{0.92}&\textbf{0.90}\\
		\bottomrule
		
	\end{tabular}
\end{table*}
\captionsetup{labelfont={color=black}}
\begin{table*}[!t]
	\renewcommand{\arraystretch}{1}
	\caption{AUC and AP results with different feature combinations (50Tr).}
	\label{combination}
	\centering
	\begin{tabular}{l| c c }
		
		\toprule
		\multicolumn{3}{c}{Cora}\\
		\midrule
		Metric&AUC&AP\\
		\hline
		Weight Averaging&0.82&0.83\\
		\hline
		Multiplication&0.58&0.61\\
		\hline
		Concatenate&\textbf{0.83}&\textbf{0.84}\\
		\bottomrule
		
	\end{tabular}
\end{table*}

\begin{table*}[!t]
	\renewcommand{\arraystretch}{1}
	\caption{AUC and AP of ablation study based on labels (50Tr).}
	\label{Ablation Experiment label 50}
	\centering
	\begin{tabular}{c c | c c c c c c c c c c c c}
		\toprule
		\multicolumn{2}{c|}{Model} & \multicolumn{2}{|c|}{Cora\_ml} & \multicolumn{2}{|c|}{Cora} & \multicolumn{2}{|c|}{Citeseer}& \multicolumn{2}{|c|}{Pubmed}& \multicolumn{2}{|c}{Bitcoin}&\multicolumn{2}{|c}{p2p-Gnutella04} \\
		\midrule
		Path-based&Community-based&AUC&AP&AUC&AP&AUC&AP&AUC&AP&AUC&AP&AUC&AP\\
		\hline
		\Checkmark&\XSolidBrush &0.80&0.83&0.72&0.73&0.74&0.76&0.68&0.70&0.88&0.88&0.78&0.72\\
		\hline
		\Checkmark&\Checkmark&\textbf{0.86}&\textbf{0.87}&\textbf{0.83}&\textbf{0.84}&\textbf{0.80}&\textbf{0.81}&\textbf{0.81}&\textbf{0.80}&\textbf{0.92}&\textbf{0.90}&\textbf{0.80}&\textbf{0.73}\\
		\bottomrule
	\end{tabular}
\end{table*}

\begin{table*}[!t]
	\renewcommand{\arraystretch}{1}
	\caption{AUC and AP of ablation study based on embedding and labels (50Tr).}
	\label{Ablation Experiment 50}
	\centering
	\begin{tabular}{c c | c c c c c c c c c c c c}
		\toprule
		\multicolumn{2}{c|}{Model} & \multicolumn{2}{|c|}{Cora\_ml} & \multicolumn{2}{|c|}{Cora} & \multicolumn{2}{|c|}{Citeseer}& \multicolumn{2}{|c|}{Pubmed}& \multicolumn{2}{|c}{Bitcoin}&\multicolumn{2}{|c}{p2p-Gnutella04}\\
		\midrule
		node embedding&node labeling&AUC&AP&AUC&AP&AUC&APP&AUC&AP&AUC&AP&AUC&AP\\
		\hline
		\XSolidBrush&\Checkmark &0.84&0.84&0.81&0.82&0.77&0.76&0.66&0.69&0.89&0.89&0.79&0.73\\
		\hline
		\Checkmark&\XSolidBrush&0.61&0.64&0.67&0.75&0.64&0.62&0.68&0.70&0.84&0.85&0.79&0.71\\
		\hline
		\Checkmark&\Checkmark&\textbf{0.86}&\textbf{0.87}&\textbf{0.83}&\textbf{0.84}&\textbf{0.80}&\textbf{0.81}&\textbf{0.81}&\textbf{0.80}&\textbf{0.92}&\textbf{0.90}&\textbf{0.80}&\textbf{0.73}\\
		\bottomrule
	\end{tabular}
\end{table*}

\section{Experiments}
Our proposed model performs the directed link prediction on six datasets. Area Under the Curve (AUC) and Average Precision (AP) are used as evaluation metrics and measure the performance with different models. {The \textbf{bolded} mark with the best experimental results and the \underline{underlined} mark with the second best results.}

{ AUC is widely used to evaluate model efficacy in link prediction tasks \cite{shang2022link,singh2022flp}. The AUC score can be interpreted and computed as the probability that a randomly chosen existing link is ranked higher by the prediction model than a randomly chosen non-existing link \cite{butun2019pattern}. The AUC score provides a single number summary of the performance of a prediction model across all possible classification thresholds, ranging from $0$ to $1$, where $1$ indicates perfect performance and $0.5$ indicates no better than random chance.

AP is another widely used metric in link prediction \cite{li2022collaborative,cai2021line}. It is particularly useful when the dataset is imbalanced, meaning there are much fewer existing links than non-existing links.
It is computed by taking the average of precision values, each of which is obtained when a true positive link is predicted, over the entire list of predicted links \cite{kumar2020link}. The resulting AP value will be between 0 and 1, with 1 indicating the perfect ranking of true positive links and 0 indicating the worst possible ranking.
}

We performed ablation studies to prove the rationality of each component. Our experiments are based on I7-11700K, RTX 3080 Ti, and 16 GB memory.
\subsection{Datatsets and Baseline Models}
{{The datasets include Cora\_ml \cite{bojchevski2017deep}, Cora \cite{shchur2018pitfalls}, Citeseer \cite{sen2008collective}, Pubmed\cite{yang2021inductive}, Bitcoin\cite{yi2022link} and p2p-Gnutella04 \cite{ren2022breadth}. The number of nodes, the number of links, the number of binary-directional links between nodes, and the proportion of binary-directional links are shown in Table \ref{dataset}. In this work, we compare with six other methods for link prediction in directed graphs, including Standard Graph AE (GAE) and Standard Graph VAE (GVAE) \cite{kipf2016variational}, Source/Target Graph AE (STGAE), Source/Target Graph VAE (STGVAE), Gravity Graph AE (GGAE) and Gravity Graph VAE (GGVAE) \cite{salha2019gravity} are chosen as baseline methods. We also compare with traditional methods CN \cite{newman2001clustering}, Salton \cite{chowdhury2010introduction}, Jaccard \cite{samad2017evaluation}, Sorenson \cite{sorensen1948method}, HPI \cite{ravasz2002hierarchical}, HDI \cite{zhou2009predicting}, LHN-1 \cite{leicht2006vertex}, PA \cite{barabasi1999emergence}, AA \cite{adamic2003friends}, RA \cite{zhou2009predicting}, LP \cite{aziz2020link}, katz \cite{katz1953new}, ACT \cite{Klein1993ResistanceD}, and Cos \cite{fouss2007random}.}} {To further analyze the performance of our model, we also compared our model's performance with DiGAE\cite {kollias2022directed}, DiGAE\_single\_layer\cite {kollias2022directed}, JPA\cite{zhang2024clusterlp}, LINE\cite{tang2015line}, Node2vec\cite{grover2016node2vec}, SEAL\cite{zhang2018link}, ARVGA\cite{pan2018adversarially}, GIC\cite{mavromatis2021graph}, LMA\cite{salha2022modularity}, MaskGAE\cite{li2023s}, S2GAE-SAGE\cite{tan2023s2gae} and ClusterLP\cite{zhang2024clusterlp}.}

\subsection{Experimental Setup}
{To demonstrate the performance of our model, we randomly chose a portion of all the connected links as the positive training samples and the rest as the positive test samples. The numbers of negative samples in the training set and the test set are the same as the number of positive samples in the corresponding sets. {We define 50Tr, 60Tr, 40Tr, and 30Tr as scenarios where 50\%, 60\%, 40\%, and 30\% of the existing links are randomly selected as positive training samples, respectively, with the remaining links used as positive test samples.}
}

{Many existing algorithms can be used for feature fusion. Without loss of generality, we use DRNL \cite{zhang2018link} as the node labeling function $\zeta _{P}$, a modularity maximization method \cite{PhysRevE.70.066111} as the community detection function $\zeta _{C}$, and DiGCL \cite{tong2021directed} as the node embedding function $\zeta _{E}$. In the experiments, we use their default settings as other approaches \cite{kipf2016variational,salha2019gravity,tong2021directed,PhysRevE.70.066111}. For DiGCL, in which default parameters have not been mentioned, we use parameters as follows: learning rate 0.001, dropout rate 0.3, weight decay 0.00001. For the GNN in FFD,
we use learning rate 0.005 and the batch size 50.
}
\subsection{Results and Analysis}

\begin{figure*}[tbp]
	\centering
	\captionsetup[subfloat]{labelsep=none,format=plain,labelformat=empty}
	\subfloat[Cora\_ml AUC]{\includegraphics[width=2.6in,height=2in]{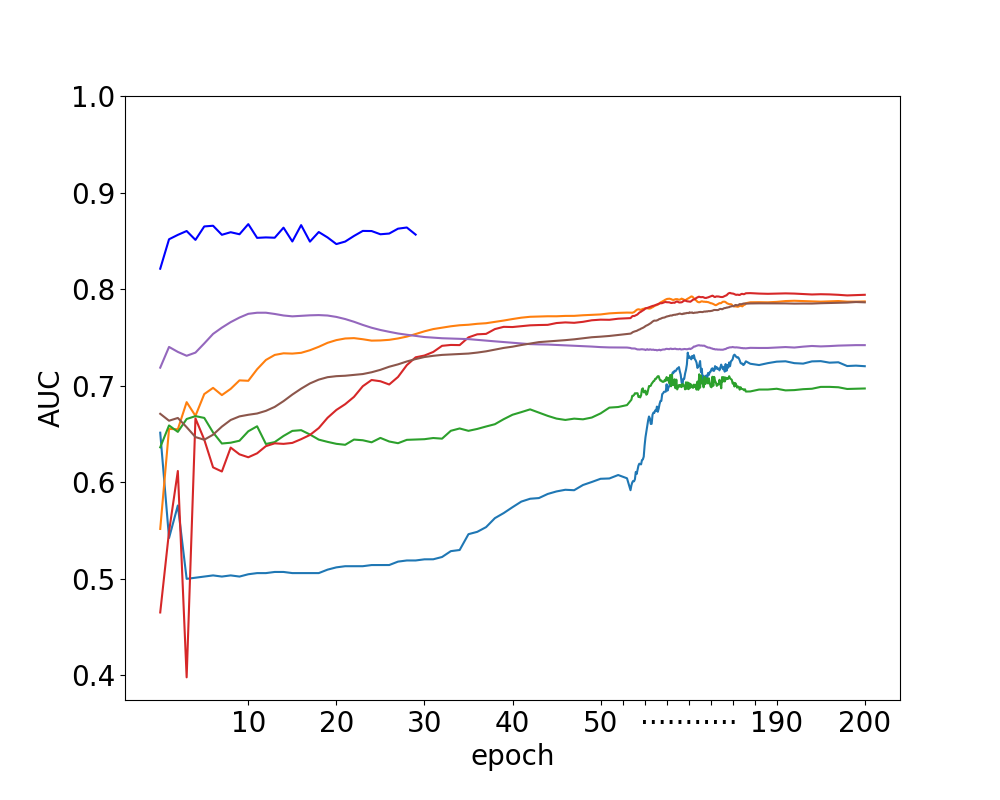}}\quad
	\subfloat[Cora AUC]{\includegraphics[width=2.6in,height=2in]{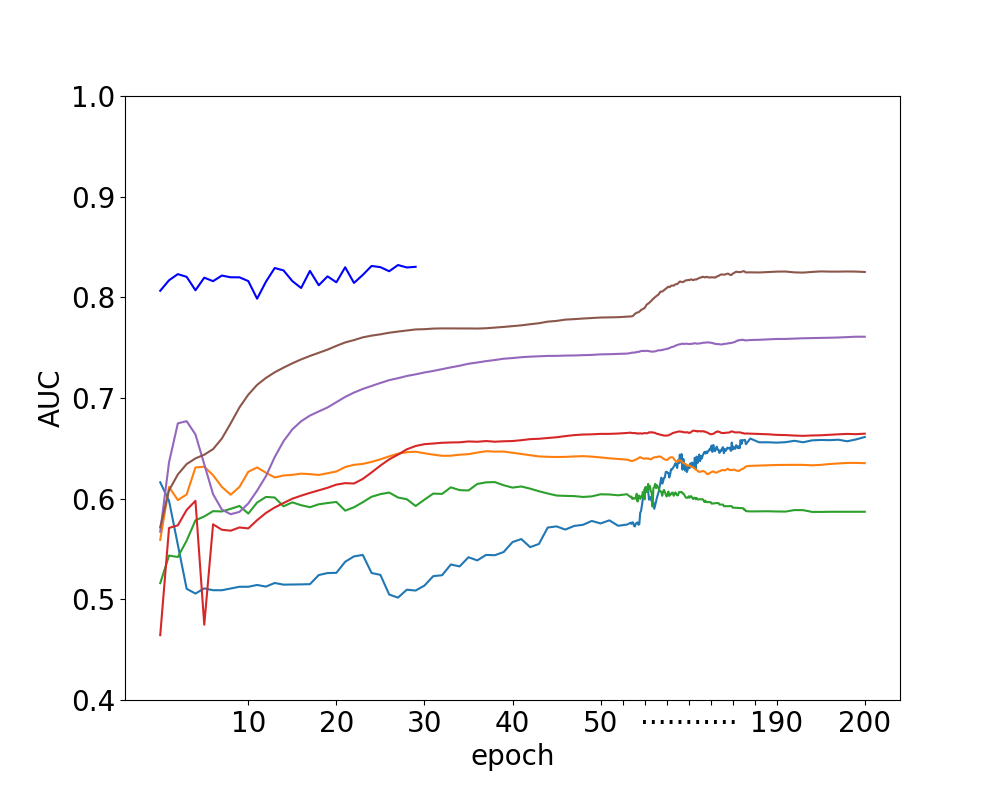}}\quad\\
	\subfloat[Citeseer AUC]{\includegraphics[width=2.6in,height=2in]{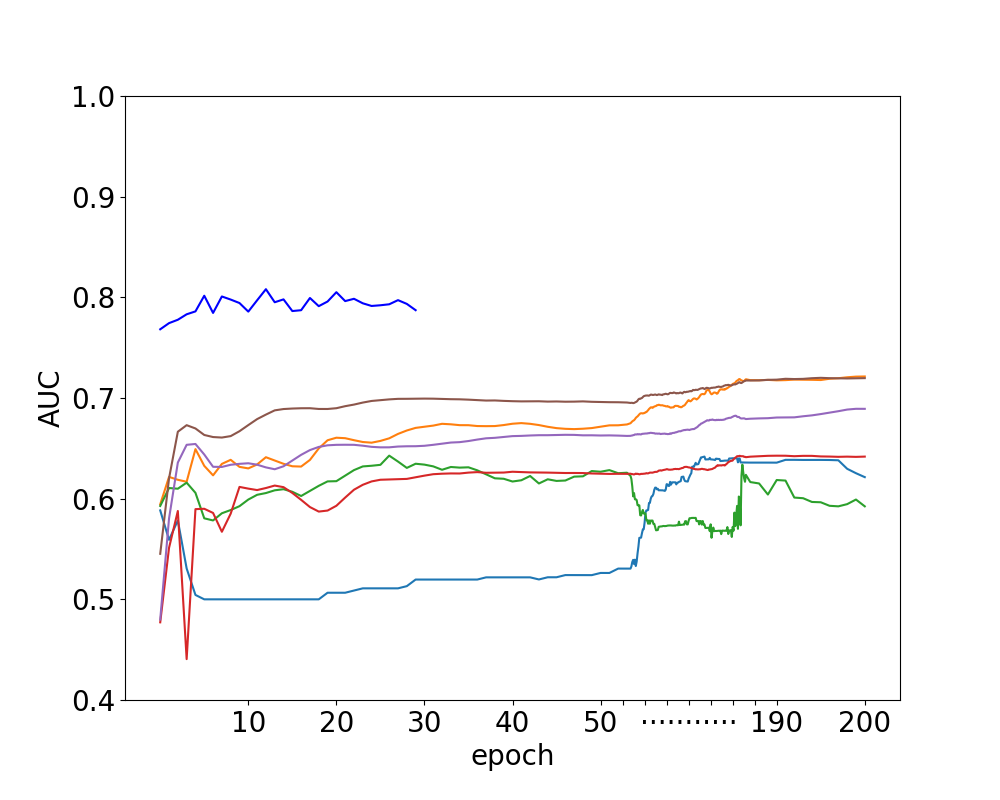}}\quad
	\subfloat[Pubmed AUC]{\includegraphics[width=2.6in,height=2in]{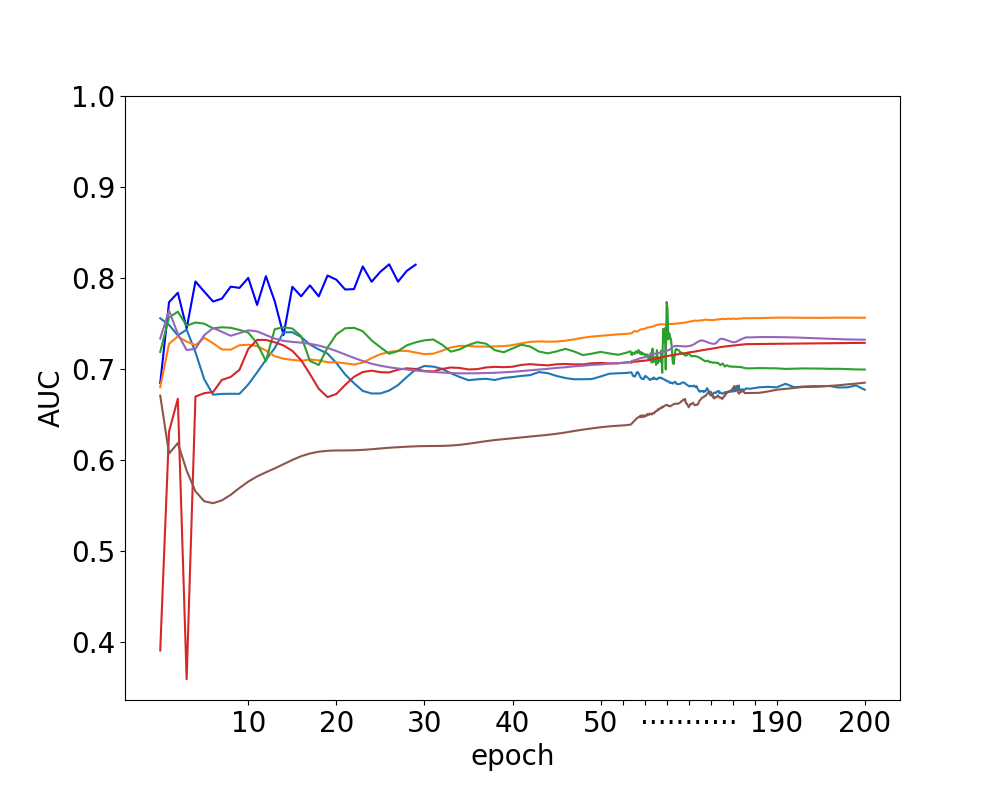}}\quad\\

	\subfloat[Cora\_ml AP]{\includegraphics[width=2.6in,height=2in]{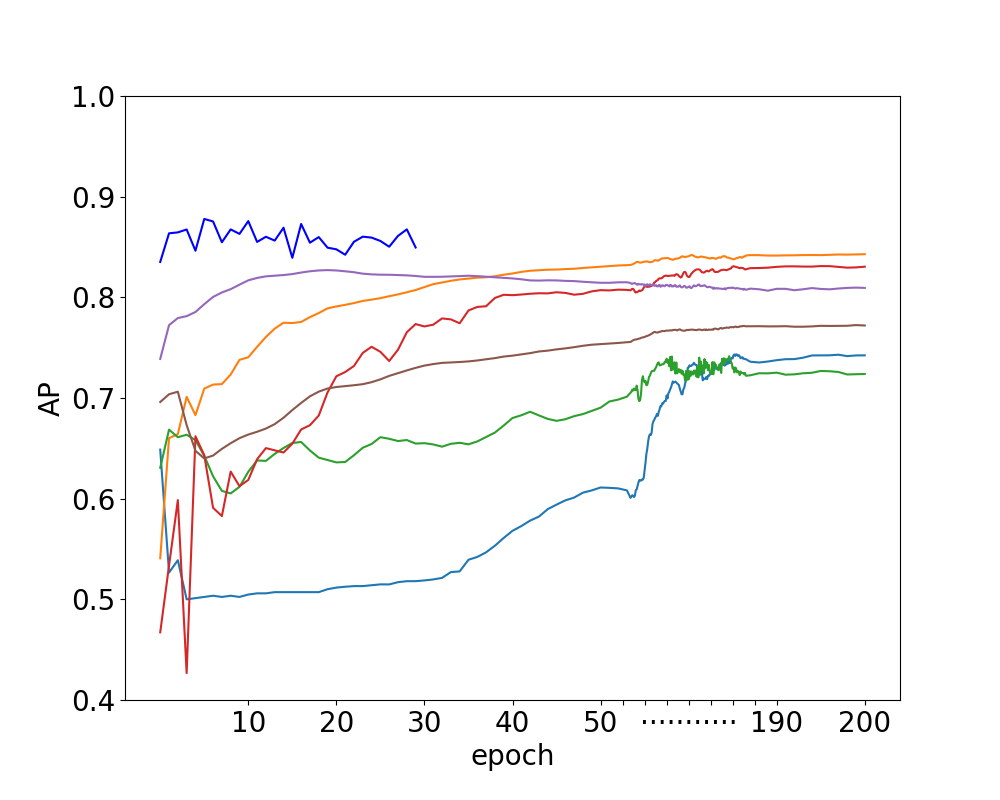}}\quad
	\subfloat[Cora AP]{\includegraphics[width=2.6in,height=2in]{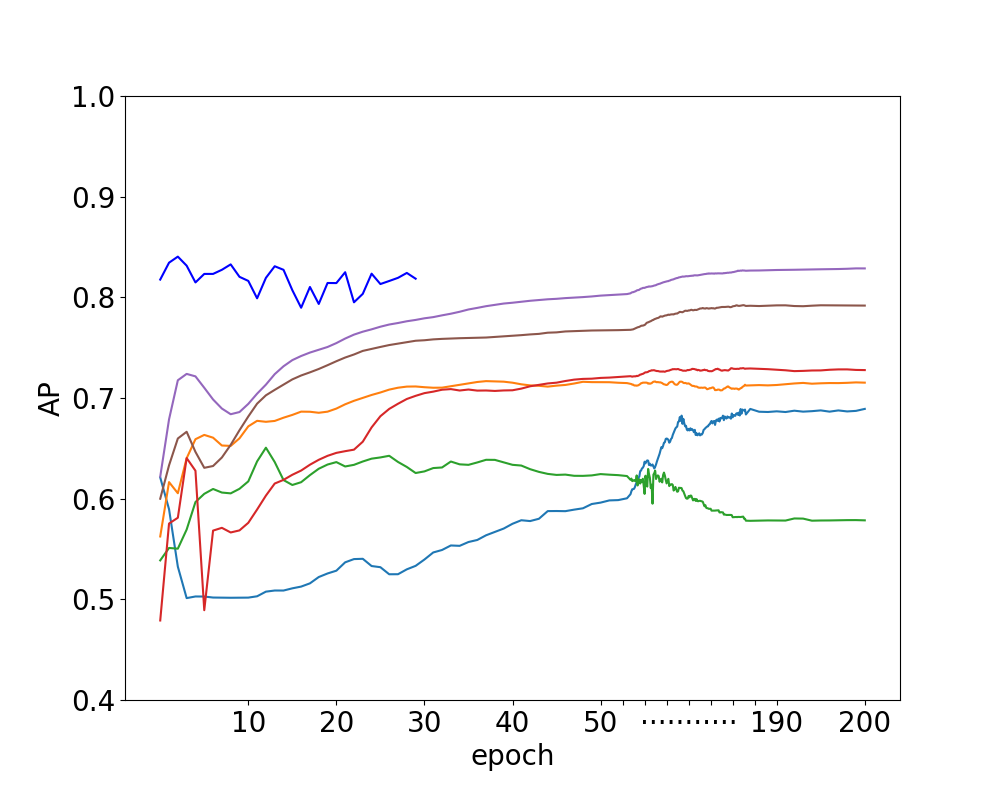}}\quad\\
	\subfloat[Citeseer AP]{\includegraphics[width=2.6in,height=2in]{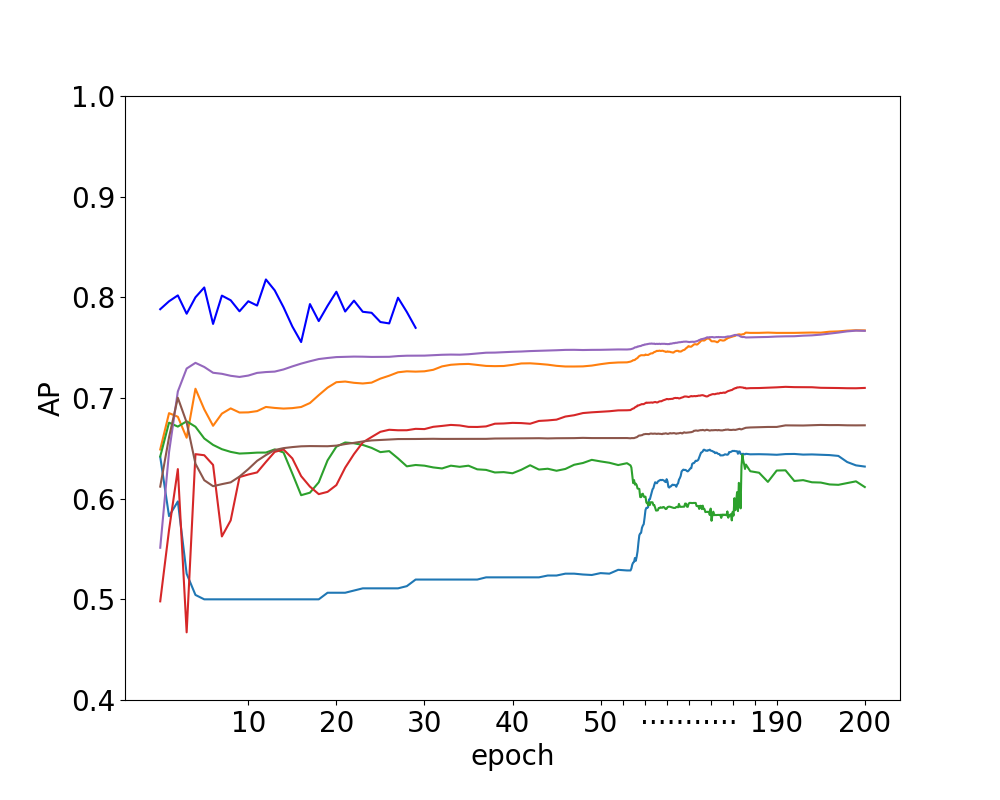}}\quad
	\subfloat[Pubmed AP]{\includegraphics[width=2.6in,height=2in]{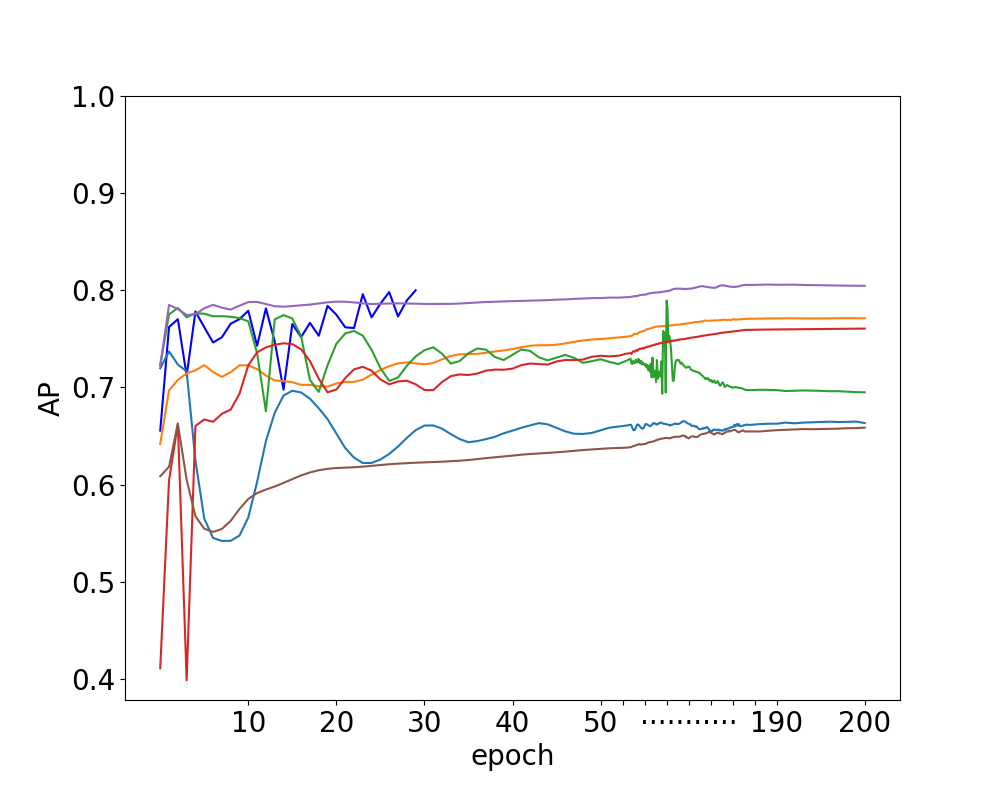}}\quad\\	
	\subfloat[FFD]{\includegraphics[width=0.5in]{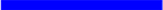}}\quad
	\subfloat[GAE]{\includegraphics[width=0.5in]{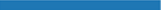}}\quad
	\subfloat[GVAE]{\includegraphics[width=0.5in]{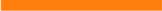}}\quad
	\subfloat[STGAE]{\includegraphics[width=0.5in]{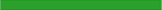}}\quad
	\subfloat[STGVAE]{\includegraphics[width=0.5in]{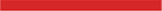}}\quad
	\subfloat[GGAE]{\includegraphics[width=0.5in]{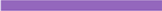}}\quad
	\subfloat[GGVAE]{\includegraphics[width=0.5in]{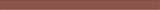}}\quad
	\caption{AUC and AP results compared to baseline methods (50Tr).}
	\label{AUC AND AP50}
\end{figure*}

\captionsetup{labelfont={color=black}}
\begin{figure*}[tbp]
	\centering
	\captionsetup[subfloat]{labelsep=none,format=plain,labelformat=empty}	
	\subfloat[K=5]{\includegraphics[width=3in]{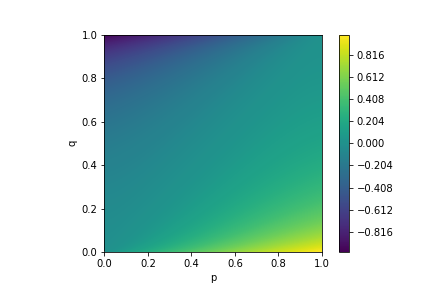}}\quad
	\subfloat[K=10]{\includegraphics[width=3in]{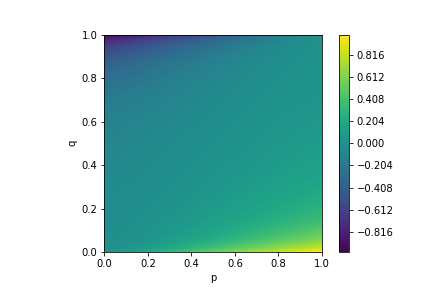}}\quad
	\subfloat[K=15]{\includegraphics[width=3in]{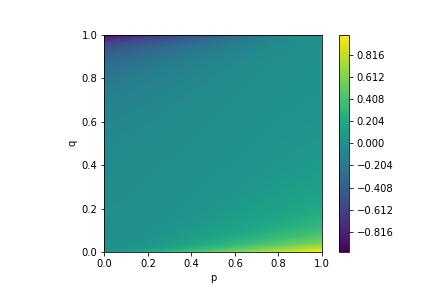}}\quad
	\subfloat[K=20]{\includegraphics[width=3in]{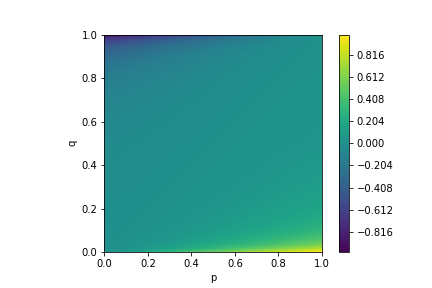}}\quad
	\caption{The values of $g(p,q, K)-1$ corresponding to different values of $p$ and $q$ when $K$ is fixed.}
	\label{pq}
\end{figure*}

\subsubsection{Directed graph Link Prediction}
{We conducted our model and six baseline models on six datasets.
In Table \ref{result60} and Table \ref{result50}, it can be seen that our model has achieved the best results in terms of AUC and AP in most cases when 60\% and 50\% of the randomly selected connected links are used in the train data respectively. Compared to the best scores achieved by other approaches,
our method demonstrates an average gain of 4.2\% in AUC and 1.2\% in AP when 60\% of the connected links were used as positive training samples. Similarly, our method exhibited an average gain of 3.6\% in AUC and 2.5\% in AP when 50\% of the connected links were used as positive training samples.

In addition, our proposed model also works well in cases with fewer training samples. In Table \ref{result40} and Table \ref{result30}, we demonstrate that, with randomly selected 40\% and 30\% of the connected links as the positive training samples, our proposed model outperforms most of the other models. Compared to the highest scores achieved by baseline approaches, our method achieves an average gain in AUC of 9.1\% and an average gain in AP of 3.6\% with 40\% of the connected links as positive training samples. Similarly, with 30\% of the connected links as positive training samples, our method achieves an average gain of 15\% in AUC and an average gain of 7.1\% in AP. Therefore,
fusing community information with node embedding information is more likely to improve the directed link prediction ability of deep learning models when the training set is smaller. {Compared to traditional methods, graph neural network-based approaches can automatically learn the complex structures and attributes of graphs, handle nonlinear relationships, capture global information, process various types of graphs, and maintain high prediction accuracy. Table \ref{resulttraditional} shows the results of the AUC comparison between our method and the traditional method, which proves that our method is superior to the traditional method. {To further analyze the performance of our model, we conducted more comparative experiments on three datasets: Cora, Citeseer and Bitcoin, against several popular algorithms. As shown in Table \ref{result50+}, the results demonstrate that our method achieved better results on the Cora and Bitcoin datasets. On the Citeseer dataset, it performed second only to the DiGAE\_single\_layer algorithm.} We also conducted experiments with different feature combinations in Eq. (\ref{label}), using weighted averaging, multiplication, and concatenation  between features. Table \ref{combination} demonstrates that relatively better results can be achieved using the concatenate method.

The Cora\_ml, Cora, Citeseer, and Pubmed datasets are classic academic citation networks with strong community structures. Such characteristics can better leverage the model capability as local and global features can complement each other. On the other hand, the Bitcoin dataset is a dynamic transaction network, and the p2p-Gnutella04 dataset is a large-scale peer-to-peer network. They both exhibit relatively weak community patterns. Consequently, the improvements on these datasets are a bit limited. }

In ablation studies, to validate the effectiveness of fusing community information, we compare the performance when 50\% of the connected links are employed as positive training samples. The results, detailed in Table \ref{Ablation Experiment label 50}, clearly demonstrate the substantial performance enhancement achieved by incorporating community information. Namely, the approach leveraging community information significantly outperforms its counterpart without community information, with an average gain of 9\% in AUC and 7\% in AP. This phenomenon underscores the crucial role of community information in improving predictive accuracy for deep learning-based directed link prediction methods.

We further explored the effectiveness of node labeling and node embedding, respectively, where node labeling includes path-based and community-based node labeling. Experiments were conducted with 50\% of connected links as positive training samples. The results of experiments with 50\% of the connected links as positive training samples are shown in Table \ref{Ablation Experiment 50}. We can find that, FFD with all modules achieved the best performance. The average AUC gain and the average AP gain with node labeling are 19\%, and are 15.8\%, respectively. The average AUC gain and the average AP gain with node embedding are 5.3\% and 4.6\%, respectively.

}
{Experiments also demonstrate that our approach converges with fewer epochs than other approaches. The experiments are carried out on four data sets with 50\% of the connected links as the positive training samples, and the AUC and the AP on the testing set at the end of each epoch are recorded and shown in Fig. \ref{AUC AND AP50}.

We also quantitively analyze $g(p,q, K)$, the left part of Eq. (\ref{eq7}) in Theorem \ref{T0}, experimentally. Without loss of generality, we assume the error rate to be zero, and plot $g(p,q, K)$+1 over different values of $p$ and $q$. From Fig. \ref{pq}, we can see that $g(p,q, K)$ increases with the increase of $p$ and the decrease of $q$, indicating that a clearer community structure makes a community detection algorithm more useful in directed link prediction, which is in accordance with our analysis.

\section{Conclusion}
Graph neural network has achieved excellent results in graph learning problems. In directed link prediction, graph neural networks usually learn the neighbors around target links or learn the global features in a directed graph. In this work, we use a feature fusion method for directed link prediction that is able to incorporate community information and line graph information with node embedding methods, namely, fusing local and global features. {The hybrid node fusion of FFD enhances the robustness of the model. At the same time, the node-link transformation empowers graph neural networks to learn and predict features of directed links directly. Experimental results demonstrated its effectiveness, and such an approach is feasible to be extended for other kinds of graphs.

In the future, besides community structures, integration of core-periphery structures \cite{wang2024core,shen2021finding} with deep learning networks may also enhance model performance. Additionally, the line graph data corresponding to large-scale graph data tends to be complex. Simplifying the corresponding line graph data while preserving information as much as possible is also a promising direction. Finally, simultaneously training related tasks to provide auxiliary information can further enhance the effectiveness of the directed link prediction model.}
}

\ifCLASSOPTIONcaptionsoff
  \newpage
\fi

\bibliographystyle{IEEEtran}

\bibliography{mybib}

\begin{IEEEbiography}[{\includegraphics[width=1in,height=1.25in,clip,keepaspectratio]{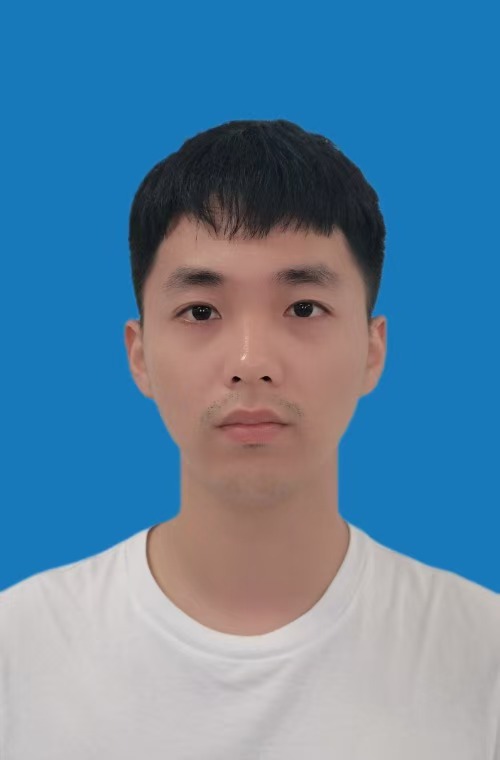}}]{Yuyang Zhang}
Yuyang Zhang received the B.Eng. degree from Taizhou University, Taizhou, China, in 2021, and the M.Eng. degrees from the Ningbo University, Ningbo, China, in 2024. His research interests include complex network analysis and machine learning.
\end{IEEEbiography}

\begin{IEEEbiography}[{\includegraphics[width=1in,height=1.25in,clip,keepaspectratio]{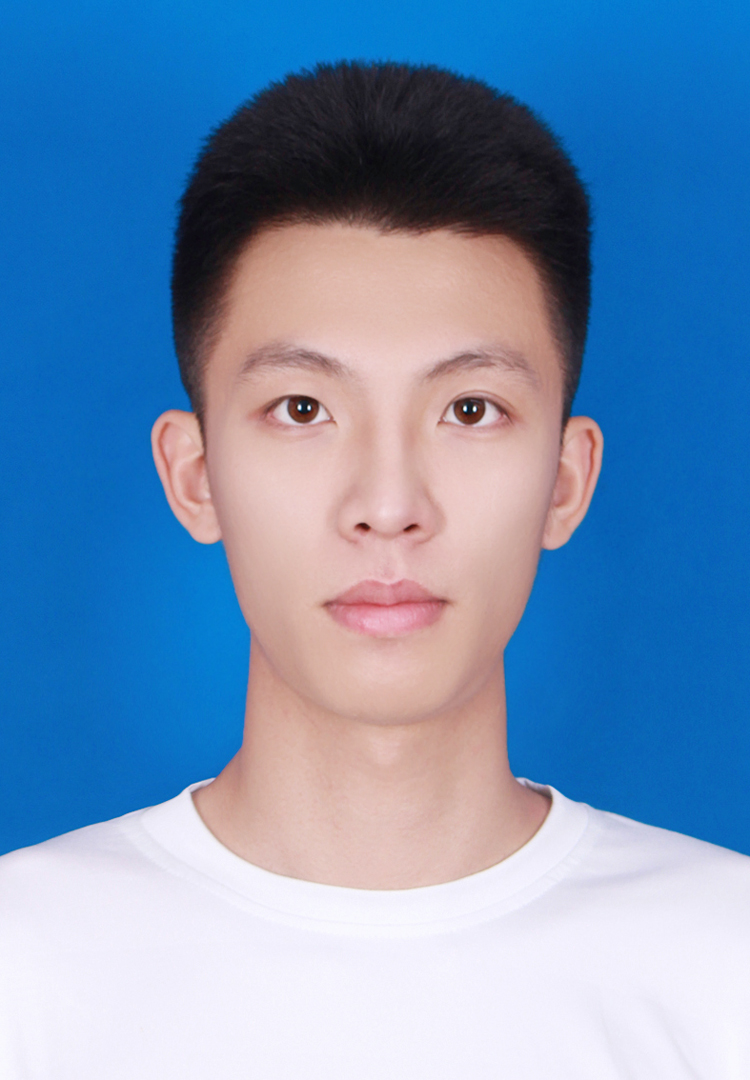}}]{Xu Shen}
received his Bachelor degree from the College of Information Science and Engineering, Ningbo University, supervised by Dr. Chengbin Peng. His research interests include data mining, graph representation learning, large language models on graphs.
\end{IEEEbiography}

\begin{IEEEbiography}[{\includegraphics[width=1in,height=1.25in,clip,keepaspectratio]{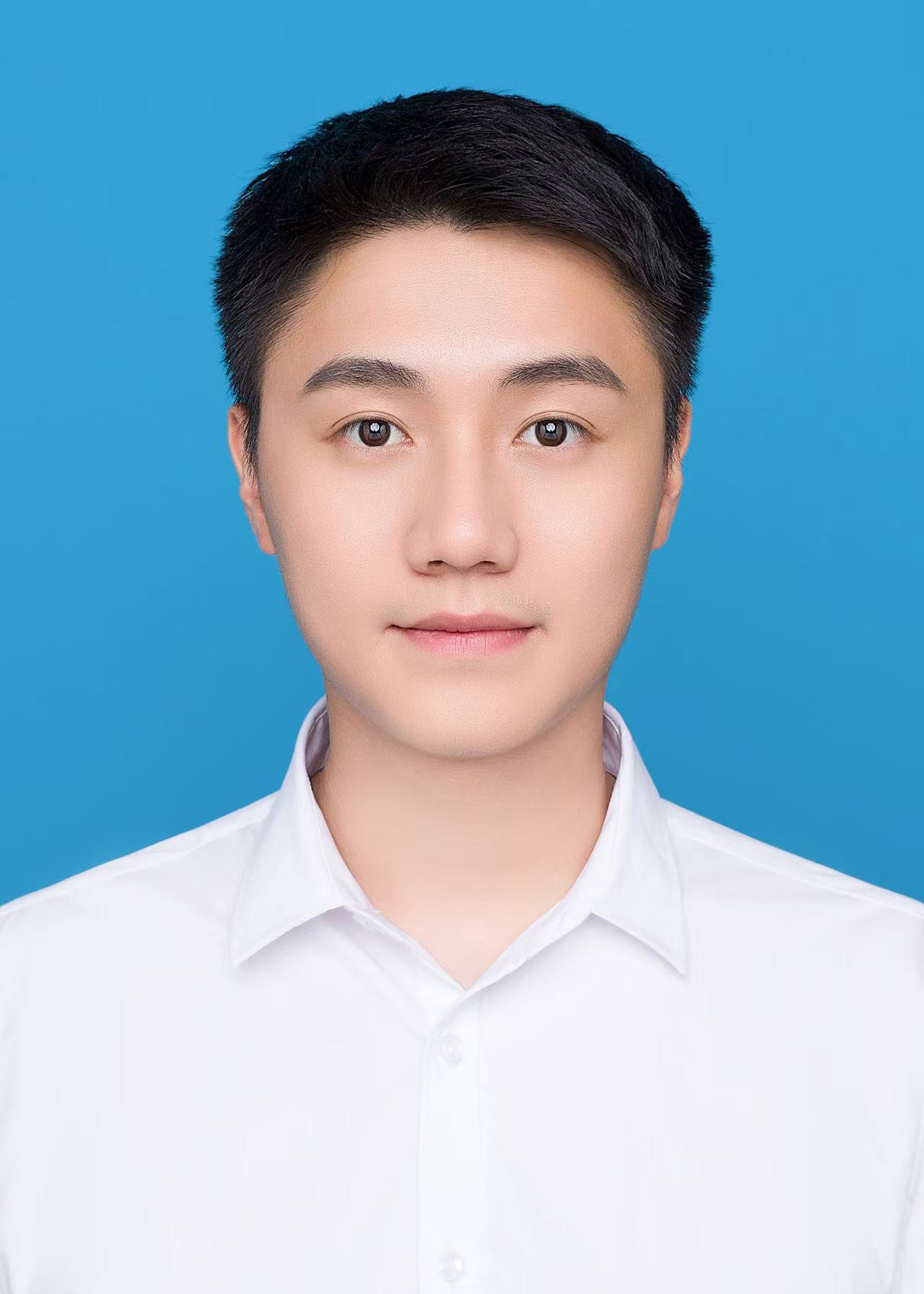}}]{Yu Xie}
received his Bachelor degree from the College of Information Science and Engineering, Ningbo University, supervised by Dr. Chengbin Peng. His interests include graph neural networks, pattern recognition, and data mining.
\end{IEEEbiography}

\begin{IEEEbiography}[{\includegraphics[width=1in,height=1.25in,clip,keepaspectratio]{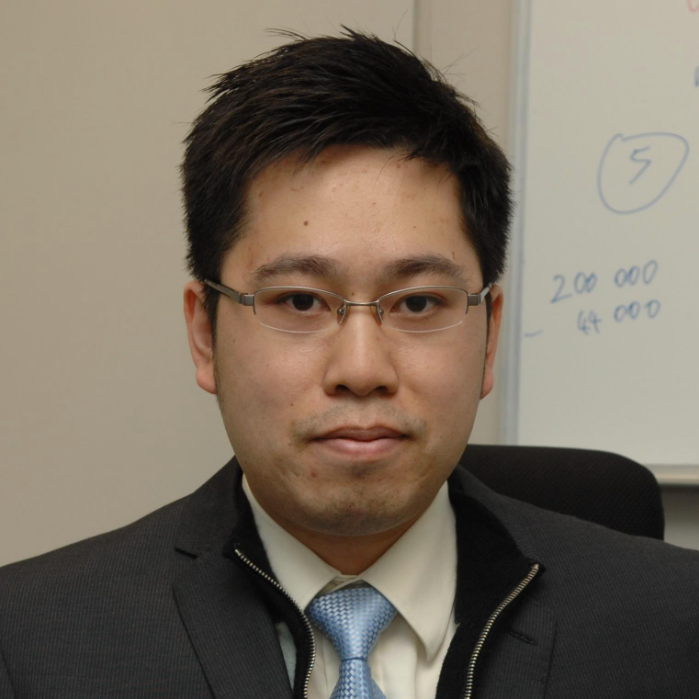}}]{Ka-chun wong}
received his B.Eng. in Computer Engineering from The Chinese University of Hong Kong in 2008. He has also received his M.Phil. degree at the same university in 2010. He received his PhD degree from the Department of Computer Science, University of Toronto in 2015. After that, he assumed his duty as assistant professor at City University of Hong Kong. His research interests include Bioinformatics, Computational Biology, Evolutionary Computation, Data Mining, Machine Learning, and Interdisciplinary Research.
\end{IEEEbiography}

\begin{IEEEbiography}[{\includegraphics[width=1in,height=1.25in,clip,keepaspectratio]{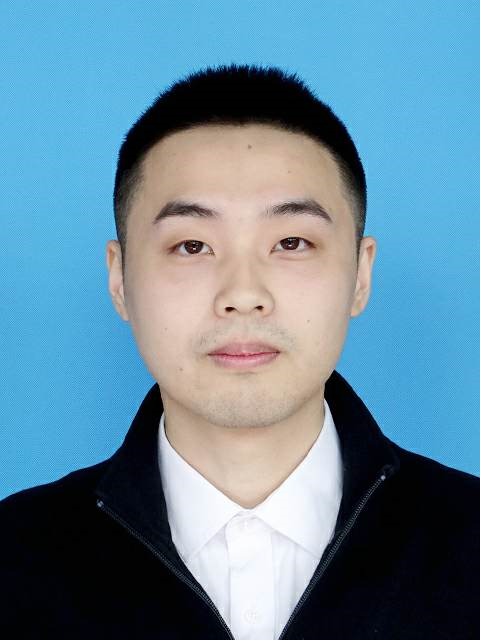}}]{Weidun Xie}
has earned his PhD in the Department of Computer Science at City University of Hong Kong. During his studies, he spent nine months as a visiting student at the University of Oxford. His research interests encompass machine learning, deep learning, and bioinformatics.
\end{IEEEbiography}

\begin{IEEEbiography}[{\includegraphics[width=1in,height=1.25in,clip,keepaspectratio]{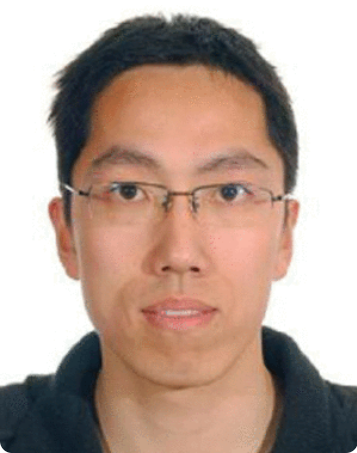}}]{Chenbin Peng (Member. lEEE)}
received the B.E. and M.S. degrees in computer science from Zhejiang University, Hangzhou, China, and the Ph.D. degree in computer science from the King Abdullah University of Science and Technology, Thuwal, Saudi Arabia, in 2015. He is currently an Associate Professor with the College of Information Science and Engineering, Ningbo University, Ningbo, China. His research interests include computer vision, complex network analysis, and semi-supervised learning.
\end{IEEEbiography}
\vfill
\end{document}